\documentclass[runningheads]{llncs}

\usepackage{eccv}

\usepackage{eccvabbrv}

\usepackage{graphicx}
\usepackage{booktabs}
\usepackage{amsmath}
\usepackage{amssymb}
\usepackage{float}
\usepackage[accsupp]{axessibility}

\usepackage{hyperref}
\usepackage{orcidlink}

\begin{document}

\title{Geometry-Aware Hyperbolic Residual-Quantized Variational Autoencoders}

\titlerunning{Geometry-Aware Hyperbolic Residual Quantization}
\author{Alessio Colombo \and Melika Ayoughi}
\authorrunning{A.~Colombo and M.~Ayoughi}
\institute{Universiteit van Amsterdam, Amsterdam, Netherlands\\
\email{alessio.colombo@student.uva.nl}, \email{m.ayoughi@uva.nl}}

\maketitle

\begin{abstract}
Residual Vector Quantization turns continuous representations into discrete, multi-level token sequences. Yet most methods operate in Euclidean space, despite the coarse-to-fine structure of the resulting codes and the latent hierarchies present in many data domains. Hyperbolic geometry offers a natural alternative for hierarchical representations, but naive hyperbolic extensions introduce geometric inconsistencies: non-associative hyperbolic addition prevents consistent residual aggregation, while standard straight-through gradient estimation ignores the geometry of the latent space.
We propose a geometry-aware hyperbolic residual quantization that addresses these issues in both the forward and backward passes. In the forward pass, \textit{Hyperbolic Residual Aggregation} restores the telescoping behavior of residual quantization on the Poincaré ball. In the backward pass, a \textit{discounted Hyperbolic Straight-Through Estimator} routes the reconstruction gradient through the quantizer as a single geometric block, avoiding unstable recursive gradient transport across residual stages. Evaluations on hierarchical prediction, recommendation, image tokenization, and neural audio coding tasks show that our method improves the stability and structural organization of hyperbolic residual codes over naive hyperbolic baselines. At the same time, we observe a clear structure–compression trade-off: Euclidean residual quantization remains preferable for pure compression, while geometry-aware hyperbolic quantization is most useful for hierarchically organized discrete latent spaces.
\end{abstract}

\section{Introduction}
\label{sec:intro}
In recent years, a growing line of work in modern generative modeling is increasingly shifting toward learning discrete representations \cite{originalvq, vqvae, finitescalarquantizationfsq, gumbelvq, vqnew}. By converting continuous signals into token sequences, vector-quantized autoencoders make it possible to apply powerful sequence models, such as autoregressive transformers \cite{lee2022rqvae, tiger} and diffusion-based architectures \cite{resgen}, to images \cite{vqvae, lee2022rqvae}, audio \cite{soundstream, encodec, audiolm}, text \cite{collapse_2}, and multi-modal data \cite{musicgen, talkplay}.
Residual vector quantization extends this idea by representing an input through multiple quantization stages: early codebooks capture coarse information, while later codebooks refine the remaining error \cite{soundstream, lee2022rqvae}. 

Most residual quantization methods, however, operate in Euclidean latent spaces \cite{soundstream, encodec, lee2022rqvae}, which are intrinsically flat, with volume growing only polynomially in the radius \cite{sala2018representation}. This is not always ideal: many data domains contain latent hierarchical structure \cite{nickel2017poincare, sala2018representation} that the multi-stage codes of residual quantization could organize given a suitable geometry. Hyperbolic geometry is a natural candidate, as its negative curvature and exponential volume growth represent hierarchies with low distortion \cite{cannon1997hyperbolic, sala2018representation, sarkar2011}. Residual quantization in hyperbolic space could thus provide a better inductive bias when the goal is not only compression but the discovery of hierarchically organized representations \cite{hrqvae, mathieu2019poincare}.

Recent works explore hyperbolic vector quantization \cite{hypervq, hvqvae, bu2025ggball} and hyperbolic residual quantization \cite{hrqvae, hyprqvae_recsys}. Yet a naive transfer to hyperbolic space introduces two geometric inconsistencies. First, during the forward pass, the residual cascade is no longer algebraically accurate. In Euclidean residual quantization, subtracting selected codewords from the residual and summing them back into the reconstruction are inverse operations. In hyperbolic space, the corresponding operation, Möbius addition, is non-associative and non-commutative \cite{ungar2009gyrovector, ganea2018hyperbolic}; so naive aggregation fails to recompose the encoder input, and the residual cascade drifts as depth increases. Second, during the backward pass, the standard straight-through estimator passes gradients through the residual quantizer as if the latent space were Euclidean. This ignores the geometry of the manifold and leads to unstable gradient flow across residual stages.

We address both issues with a geometry-aware hyperbolic residual-quantized variational autoencoder (GHRQ-VAE) that repairs the forward and backward passes, stabilizing training at depth and enabling the representation of latent hierarchical structures.

Our contributions are threefold. (i) We identify and formalize the forward and backward geometric inconsistencies that arise when residual quantization is naively lifted to hyperbolic space. (ii) We introduce a geometry-aware hyperbolic residual quantizer that restores consistent residual aggregation and provides stable block-level gradient routing on the Poincaré ball. (iii) We evaluate the method across hierarchical prediction, recommendation, image tokenization, and neural audio coding. To the best of our knowledge, this is the first application of hyperbolic RQ-VAEs to image and audio tasks. Our results reveal that while Euclidean methods remain preferable for pure signal compression and raw reconstruction fidelity, geometry-aware hyperbolic residual quantization provides more stable training and yields more structurally organized hyperbolic residual codes, especially compared with naive hyperbolic baselines.


\section{Related Work}
\label{sec:related}
\paragraph{Vector and Residual Quantization.}
Vector quantization has a long history in signal processing for lossy compression \cite{originalvq}, and was revived in deep learning by the VQ-VAE \cite{vqvae}, which introduced a discrete bottleneck into the autoencoder framework (we review its mechanics in \S\ref{sec:bg_vq}). A persistent obstacle is codebook collapse \cite{collapse_1, collapse_2}, where only a few codewords are ever selected; remedies range from codebook resets and EMA updates to alternative bottlenecks \cite{hqvae, gumbelvq, finitescalarquantizationfsq}. Residual Vector Quantization (RVQ) extends single-stage VQ by quantizing in multiple successive stages \cite{soundstream}, constructing a virtual codebook of exponential capacity without growing the memory footprint or sequence length. RQ-VAE has since been applied across image generation \cite{lee2022rqvae, resgen, progic, evotok}, audio codecs \cite{soundstream, encodec}, audio generation \cite{audiolm, musicgen}, and generative recommendation \cite{tiger, hyprqvae_recsys}, all sharing an inherent hierarchy: truncating the code tuple at any depth yields a coarser but coherent approximation. Yet all existing RVQ methods operate in Euclidean space, whose flat geometry and polynomial volume growth are mismatched with the hierarchical structure that residual quantization induces.

\paragraph{Hyperbolic Representation Learning.}
In hyperbolic space, the volume of a geodesic ball grows exponentially with radius, in stark contrast to the polynomial growth of Euclidean space \cite{cannon1997hyperbolic}. Sala et al.\ \cite{sala2018representation} showed that any weighted tree with $n$ nodes embeds into two-dimensional hyperbolic space with arbitrarily low distortion, whereas Euclidean space requires $\mathcal{O}(n)$ dimensions, a gap that continues to guide the design of hierarchy embeddings \cite{ayoughi2025designing}. Hyperbolic representation learning was pioneered by Nickel and Kiela \cite{nickel2017poincare, nickel2018lorentz}, and Ganea et al.\ \cite{ganea2018hyperbolic} formalized neural network operations on the ball via Möbius gyrovector algebra \cite{ungar2009gyrovector}; subsequent work moved computation fully onto the manifold \cite{fullyhyperbolic, hcnn, poincareresnet}. Hyperbolic networks have since been applied across word embeddings \cite{tifrea2019poincare, hypertext, dhingra2018embedding}, graph learning \cite{chami2019hyperbolic, liu2019hgnn, hgnnsurvey, li2024hyperbolic}, computer vision \cite{khrulkov2020hyperbolic, atigh2022hyperbolic, ermolov2022hvt, mettes2024hyperbolic, liu2020hyperbolic, peng2021hyperbolic}, continual and incremental learning \cite{ayoughi2025continual, hindel2024taxonomy, sur2025hyperbolic}, and vision-language and language models \cite{pal2025compositional, helm, he2025survey, ibrahimi2024intriguing}. Poincaré VAE \cite{mathieu2019poincare} and related hyperbolic generative models \cite{nagano2019wrapped, skopek2020mixed, apovae, hypdae} showed that hierarchical structure emerges in hyperbolic latents without supervision, while supervised approaches embed known label hierarchies through entailment cones \cite{ganea2018cones, dhall2020hierarchical, chen2020hyperbolic} and ideal boundary prototypes \cite{atigh2021busemann, berg2025multiprototype}, at any level of granularity \cite{atigh2026hyperbolic}. Several works combine hyperbolic geometry with vector quantization \cite{hypervq, hvqvae} and residual vector quantization \cite{hrqvae,hyprqvae_recsys}. Because these methods retain Euclidean gradient transport and a non-associative aggregation, we make both geometrically aware.

\section{Preliminaries}
\label{chap:background}
\paragraph{Vector-Quantized Variational Autoencoders.}
\label{sec:bg_vq}
The VQ-VAE \cite{vqvae} learns a discrete latent representation of continuous data. Given an input signal $x$, an encoder network $E$ maps it to a continuous latent representation $z_e = E(x) \in \mathbb{R}^d$. A trainable codebook $C = \{c_1, \dots, c_K\} \subset \mathbb{R}^d$ is then used to discretize this vector via the nearest codeword under the Euclidean distance, $q(z_e) = c_k$ with $k = \operatorname{argmin}_j \| z_e - c_j \|^2_2$. The selected vector is passed to a decoder $G$ that reconstructs $\hat{x} = G(q(z_e))$. Since the nearest-neighbor assignment is non-differentiable, the VQ-VAE uses the Straight-Through Estimator (STE) \cite{straight_through_estimator}: the forward value is computed while the difference $q(z_e) - z_e$ is held under the stop-gradient operator $\operatorname{sg}[\cdot]$,
\begin{equation}
    \label{eq:bg_ste}
    \hat{z}_{\text{STE}} = z_e + \operatorname{sg}\big[q(z_e) - z_e\big].
\end{equation}
This evaluates to $q(z_e)$ in the forward pass, yet because the stop-gradient term is treated as a constant, its Jacobian reduces to $\partial \hat{z}_{\text{STE}} / \partial z_e = I$, so the decoder gradient is passed unaltered to the encoder.

\paragraph{Residual Vector Quantization.}
\label{sec:bg_rvq}
RVQ extends the single-stage paradigm by quantizing the latent representation in $N$ successive stages, each with its own codebook \cite{soundstream}. Let $r_0 = z_e$ denote the encoder output. At stage $i$ the current residual $r_{i-1}$ is quantized to a codeword $q_i$, and the residual for the next stage is obtained by subtracting the selected codeword, $r_i = r_{i-1} - q_i$. The STE of Eq.~\ref{eq:bg_ste} is applied independently at every stage, and the final representation is the sum of the selected codewords, $\hat{z} = \sum_{i=1}^{N} q_i$. This coarse-to-fine decomposition constructs a virtual codebook of effective size $K^N$ and induces a natural hierarchy: early stages capture coarse, global structure while later stages encode progressively finer detail. All stages are trained jointly with the objective
\begin{equation}
    \label{eq:rqvae_loss}
    \mathcal{L} = \underbrace{\lVert x - \hat x\rVert_2^2}_{L_{\text{rec}}}
    \;+\; \sum_{i=1}^{N}\Big(
    \underbrace{d\big(\operatorname{sg}[r_{i-1}],\, q_i\big)^2}_{\text{codebook}}
    \;+\; \beta\,\underbrace{d\big(r_{i-1},\, \operatorname{sg}[q_i]\big)^2}_{\text{commitment}}
    \Big),
\end{equation}
where the reconstruction term $L_{\text{rec}}$ is decoded from the aggregate code as $\hat x = G(\hat z)$, $\beta$ weights the commitment term, and $d(\cdot,\cdot)$ is a distance on the latent space. The Euclidean distance recovers the standard RQ-VAE loss; replacing it with the geodesic distance and the hyperbolic aggregate of \S\ref{sec:residagg} gives the hyperbolic counterpart we adopt.

\paragraph{Hyperbolic Geometry.}
\label{sec:bg_hyperbolic}
Hyperbolic space is a Riemannian manifold of constant negative curvature whose geodesic-ball volume grows exponentially with radius \cite{cannon1997hyperbolic}.
We adopt the Poincar\'e ball, which is well suited to gradient-based learning because its operations admit closed forms \cite{ganea2018hyperbolic}. For a curvature parameter $c>0$, the Poincar\'e ball $\mathbb{D}_c^d = \{ x \in \mathbb{R}^d : c\|x\|^2 < 1 \}$ is equipped with the conformal metric $g_x^c = (\lambda_x^c)^2 g^E$, with conformal factor $\lambda_x^c = 2/(1 - c\|x\|^2)$; \emph{conformal} means that the metric is a positive pointwise rescaling of the Euclidean metric $g^E$, so angles are preserved but lengths are not. The setting $c=0$ recovers Euclidean space. The conformal factor is decisive: $\lambda_0^c=2$ at the origin but diverges, $\lambda_x^c\to\infty$, toward the boundary, which both grants expressive power and makes computations numerically delicate there. In the Poincaré ball, the analogue of addition is the M\"obius addition
\begin{equation}
    \label{eq:bg_mobius_add}
    x \oplus_c y = \frac{\big(1 + 2c\langle x, y\rangle + c\|y\|^2\big)x + \big(1 - c\|x\|^2\big)y}{1 + 2c\langle x, y\rangle + c^2\|x\|^2\|y\|^2},
\end{equation}
with M\"obius subtraction $x \ominus_c y = x \oplus_c (-y)$. $\oplus_c$ is neither commutative nor associative. The failure of commutativity is captured by the \emph{gyration} operator
\begin{equation}
    \label{eq:bg_gyration}
    \operatorname{gyr}[u, v]\,w = \ominus(u \oplus_c v) \oplus_c \big(u \oplus_c (v \oplus_c w)\big),
\end{equation}
an automorphism of the ball that acts as a rotation: $u \oplus_c v = \operatorname{gyr}[u, v](v \oplus_c u)$. The induced geodesic distance is $d_{\mathbb{D}_c}(x, y) = \tfrac{2}{\sqrt{c}}\tanh^{-1}(\sqrt{c}\|(-x) \oplus_c y\|)$.
At every point $x$ the tangent space $T_x\mathbb{D}_c^d$ is a local Euclidean linearization of the manifold; movement between the manifold and a tangent space is mediated by the mutually inverse exponential and logarithmic maps, which at the origin take the radial forms $\exp_0^c(v) = \tanh(\sqrt{c}\|v\|)\tfrac{v}{\sqrt{c}\|v\|}$ and $\log_0^c(y) = \tanh^{-1}(\sqrt{c}\|y\|)\tfrac{y}{\sqrt{c}\|y\|}$. A tangent vector at $x$ cannot be directly compared with one at $y$; parallel transport carries it along the connecting geodesic while preserving its Riemannian norm,
\begin{equation}
    \label{eq:bg_pt}
    P_{x \to y}^c(v) = \frac{\lambda_x^c}{\lambda_y^c}\,\operatorname{gyr}[y, -x]\,v,
\end{equation}
which both rescales the vector by the ratio of conformal factors and rotates it through the gyration. Finally, because the metric rescales the inner product by $(\lambda_x^c)^2$, the Riemannian gradient relates to the Euclidean one by $\nabla^R f(x) = (\lambda_x^c)^{-2}\nabla^E f(x)$. A correct backward pass on the ball therefore converts Euclidean gradients to Riemannian ones, transports them between tangent spaces via Eq.~\ref{eq:bg_pt}, and converts back; as $\lambda_x^c$ diverges near the boundary, these conversions can amplify gradient magnitudes substantially.

\section{Method}
\label{chap:methods}
\label{sec:ghrq}
Geometry-aware Hyperbolic Residual Quantization (GHRQ) lifts residual quantization onto the Poincar\'e ball through two independent repairs that compose into a single quantizer: Hyperbolic Residual Aggregation (HRA), which provides an algebraically correct coarse-to-fine decomposition in the forward pass (\S\ref{sec:residagg}), and a block-level gradient routing, which uses a single discounted Hyperbolic Straight-Through Estimator (d-HSTE) step to send the gradient back to the encoder during the backward pass (\S\ref{sec:blockste}).


\subsection{Hyperbolic Residual Aggregation}
\label{sec:residagg}
\label{sec:newmethod}
Residual quantization rests on a single algebraic guarantee: the rule that \emph{removes} a code from the running residual and the rule that \emph{re-assembles} the codes into the reconstruction must be exact inverses. In Euclidean space this holds for free. Writing $r_0=z_e$ for the encoder point and $\hat z_i$ for the aggregate of the first $i$ codes, the update $r_i=r_{i-1}-q_i$ and the sum $\hat z=\sum_i q_i$ are mutually inverse because addition is commutative and associative; the residual entering each stage therefore equals the part of $z_e$ not yet captured by the earlier codes (the \emph{true residual} $R_i^{\text{true}}=z_e-\hat z_i=r_i$) and the decomposition \emph{telescopes}, recomposing $z_e$ up to the final, unquantized residual, $\hat z+r_{N}=z_e$. Each codebook is thus fitted to the reconstruction error left by its predecessors, which is the entire purpose of the coarse-to-fine cascade.

On the Poincar\'e ball this guarantee is no longer automatic. M\"obius addition is neither commutative nor associative, so the naive lift of the recursion, replacing the subtraction by a \emph{right} M\"obius subtraction and the sum by a left-associated M\"obius addition,
\begin{equation}
    \label{eq:prev_residual}
    r_i = r_{i-1}\oplus_c(-q_i),
    \qquad
    \hat z = \big(\cdots(q_1\oplus_c q_2)\oplus_c\cdots\big)\oplus_c q_{N},
\end{equation}
no longer inverts itself. The tracked residual drifts away from the true residual, compounds with depth, and the codes no longer recompose to $z_e$. 

\paragraph{The HRA convention.} To solve this, we choose the residual and aggregation rules so that they cancel \emph{by construction}, using the one cancellation law the gyrogroup does provide. The left-cancellation law
\begin{equation}
    \label{eq:left_cancel}
    a \oplus_c\big((-a)\oplus_c b\big) = b
\end{equation}
states that adding $a$ on the left exactly undoes subtracting $a$ on the left. Pairing a \emph{left} M\"obius subtraction in the residual update with a \emph{reverse-nested} (right-associated) aggregation,
\begin{equation}
    \label{eq:hrq_residual}
    r_i = (-q_i)\oplus_c r_{i-1},
    \qquad
    \hat z = q_1\oplus_c\big(q_2\oplus_c(\cdots\oplus_c q_{N})\big),
\end{equation}
matches each subtraction to precisely the addition that inverts it. We call this pairing the \emph{Hyperbolic Residual Aggregation} (HRA) convention. 

\paragraph{Exact telescoping.} The HRA convention inverts the cascade stage by stage. The first stage gives $r_1=(-q_1)\oplus_c z_e$, which Eq.~\ref{eq:left_cancel} (taking $a=q_1$, $b=z_e$) inverts as $q_1\oplus_c r_1=z_e$. The same identity holds at every stage, $q_i\oplus_c r_i=r_{i-1}$, so unrolling the recursion recomposes the encoder point without error, $q_1\oplus_c(q_2\oplus_c(\cdots\oplus_c(q_{N}\oplus_c r_{N})))=z_e$. Dropping the final unquantized residual $r_{N}$ leaves the reconstruction $\hat z$ of Eq.~\ref{eq:hrq_residual}, which coincides with $z_e$ up to that tail. This is exactly the Euclidean telescoping property, now recovered on the ball.

\paragraph{The residual mismatch collapses to a pure rotation.} On the ball the true residual is $R_i^{\text{true}} := (-\hat z_i)\oplus_c z_e$, the M\"obius left-difference satisfying $\hat z_i\oplus_c R_i^{\text{true}}=z_e$, which in general differs from the tracked residual $r_i$. Under HRA, however, repeatedly applying the gyration form of left cancellation relates the two by a composition of gyrations,
\begin{equation}
    \label{eq:hrq_gyration}
    R_i^{\text{true}} = \Gamma_i\, r_i,
    \qquad
    \Gamma_i = \prod_{k=1}^{i-1}\operatorname{gyr}\!\big[q_k,\,u_{k+1}\big],
\end{equation}
Each factor is a norm-preserving rotation of the ball about the origin, so $\Gamma_i$ rotates the residual's \emph{direction} while adding \emph{zero} magnitude error, $\|R_i^{\text{true}}\| = \|r_i\|$. The tracked residual thus carries the magnitude of the true reconstruction error at every depth (proof in Appendix~\ref{app:hra_rotation}), unlike the naive convention, whose mismatch is a drift that corrupts this magnitude and compounds with depth.

Figure~\ref{fig:residagg} makes the contrast concrete on the Poincar\'e disk, where HRA aggregation lands on $z_e$ up to the tail $r_{N}$ while the naive order drifts away.

\begin{figure}[t]
    \centering
    \includegraphics[width=0.44\textwidth]{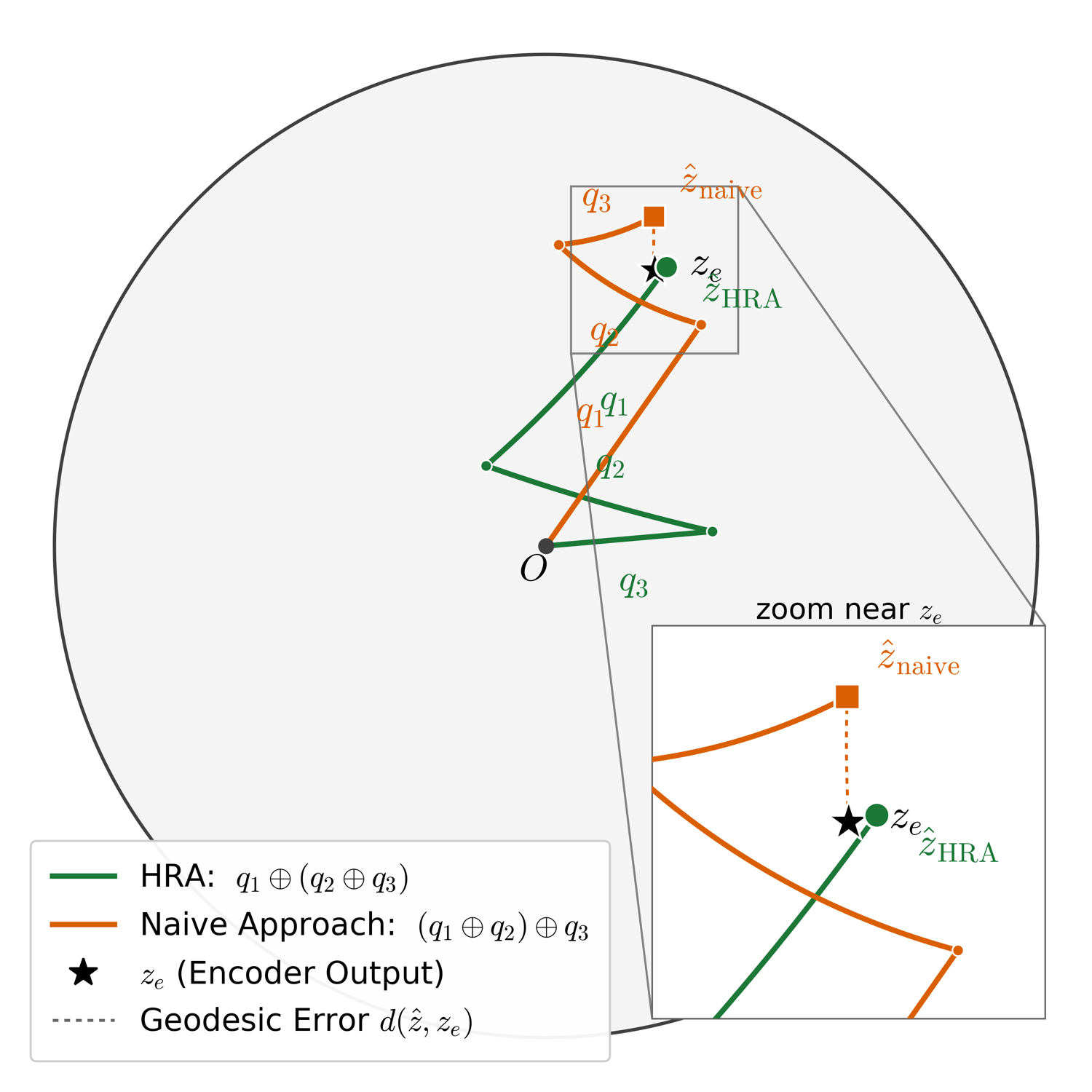}
    \caption{Residual aggregation of the same codewords under the two conventions of \S\ref{sec:residagg}, on the Poincar\'e disk. An encoder point $z_e$ (star) is quantized into coarse-to-fine codes $q_1,q_2,q_3$, and each aggregation is drawn as a chain of M\"obius hops from the origin $O$. The HRA reverse-nested order $q_1\oplus_c(q_2\oplus_c q_3)$ (green circle, Eq.~\ref{eq:hrq_residual}) reaches $\hat z_{\mathrm{HRA}}$, coinciding with $z_e$ up to the final residual. The naive left-associated order $(q_1\oplus_c q_2)\oplus_c q_3$ (orange square, Eq.~\ref{eq:prev_residual}) misplaces the gyration factors and drifts to $\hat z_{\mathrm{naive}}$; the inset zooms on $z_e$. The drift is small for shallow, well-quantized residuals, and compounds with depth and proximity to the ball boundary.}
    \label{fig:residagg}
\end{figure}

\subsection{Block-Level Gradient Routing with a Discounted HSTE}
\label{sec:hste}
\label{sec:blockste}
We repair the backward pass with two mechanisms. The first is the \emph{Discounted Hyperbolic Straight-Through Estimator} (d-HSTE), a single-step surrogate that transports one gradient between two points on the Poincar\'e ball while respecting its geometry. The second is a \emph{block-level gradient routing} strategy that, via stop-gradients, decouples the residual cascade and applies the d-HSTE exactly once. This transports a single, boundary-stable reconstructed gradient from the aggregate reconstruction $\hat z$ directly to the encoder output $z_e$, bypassing all intermediate codes $q_i$ ($i>1$) and residuals $r_i$ ($i>0$). The result is a depth-independent gradient to the encoder that parallels the Euclidean RQ-VAE behavior.

\paragraph{Discounted Hyperbolic Parallel Transport.}
In the forward pass, the d-HSTE acts as the identity mapping: $\mathrm{d\text{-}HSTE}(z_e,q) = q$. It only modifies the backward pass, where it assigns the encoder point a surrogate Jacobian in place of the non-differentiable nearest-neighbor assignment. Given a Euclidean gradient $g_q := \partial L/\partial q$ at a code $q$, exact transport to $z_e$ requires three Riemannian operations: (i) conversion to a Riemannian gradient at $q$ by dividing by $(\lambda^c_q)^2$; (ii) parallel transport along the geodesic from $q$ to $z_e$, scaling by $\lambda^c_q/\lambda^c_{z_e}$ and rotating by $\operatorname{gyr}[z_e,-q]$; and (iii) conversion back to a Euclidean gradient at $z_e$ via multiplication by $(\lambda^c_{z_e})^2$. We skip the third step, since $(\lambda^c_{z_e})^2$ diverges as $z_e$ approaches the boundary. Steps (i)--(ii) alone define a \emph{discounted parallel transport} $\widetilde P^c_{q\to z_e}$, the surrogate derivative the estimator assigns to the encoder point:
\begin{equation}
    \label{eq:hste_backward}
    \frac{\partial L}{\partial z_e}
    = \widetilde P^c_{q\to z_e}\,g_q
    = \frac{1}{\lambda^c_q\,\lambda^c_{z_e}}\,\operatorname{gyr}[z_e,-q]\,g_q .
\end{equation}
Because $q$ closely approximates $z_e$, evaluating the gyration via its standard formulation is prone to catastrophic cancellation; we instead use a numerically stable, exact reformulation of the gyration. Writing $\operatorname{gyr}[z_e,-q]\,v = v + 2\,(a\,z_e - b\,q)/d$ and setting the quantization error $\delta := q - z_e$, both the denominator $d$ and the coefficient gap $a-b$ otherwise subtract two nearly identical $O(1)$ quantities; cancelling these analytically gives the equivalent forms
\begin{align}
    \label{eq:gyr_stable}
    d &= (1-c\|z_e\|^2)^2 - 2c\,(1-c\|z_e\|^2)\,\langle z_e,\delta\rangle + c^2\|z_e\|^2\|\delta\|^2,\notag\\
    a-b &= -c\,(1-c\|z_e\|^2)\,\langle\delta,v\rangle - c^2\langle z_e,v\rangle\,\|\delta\|^2 + 2c^2\langle z_e,\delta\rangle\,\langle\delta,v\rangle ,
\end{align}
expressed through small terms of comparable magnitude (with $a\,z_e-b\,q=\tfrac{a-b}{2}(q+z_e)-\tfrac{a+b}{2}\delta$ and $a+b$ computed directly). This is mathematically identical to the closed form (Appendix~\ref{app:gyration}) but stays finite at the boundary, preserving expressiveness in high-curvature regions.

\paragraph{Block-Level Gradient Routing.}
Stacking d-HSTE steps naively would still be unstable: unlike the Euclidean case, the per-stage residual Jacobians $A_i=\partial r_i/\partial r_{i-1}$ do not vanish on the ball (the two differentials of $\oplus_c$ differ by $c\|q_i-r_{i-1}\|^2/\gamma_i$ on the directions orthogonal to $\operatorname{span}\{r_{i-1},q_i\}$, so $A_i\neq0$ unless the stage quantizes exactly), so reconstruction and commitment gradients accumulate across all $N$ stages and diverge as the residuals approach the boundary (Appendix~\ref{app:leak}). To circumvent this recursive instability, we decouple the intermediate residuals from the computational graph. Applying a stop-gradient to every intermediate residual ($r_i\leftarrow\operatorname{sg}[r_i]$ for $i\ge 1$) lets the codes propagate their forward values into the aggregation (Eq.~\ref{eq:hrq_residual}) without backpropagating the reconstruction gradient through the cascade. Instead, the encoder receives the reconstruction gradient via a single d-HSTE step from the aggregate reconstruction $\hat z = q_1\oplus_c(q_2\oplus_c(\cdots\oplus_c q_{N}))$ directly to $r_0=z_e$. The full decoder gradient $g_{\hat z}:=\partial L_{\text{rec}}/\partial\hat z$ is transported using Eq.~\ref{eq:hste_backward} with $q=\hat z$, yielding the block-level estimator
\begin{equation}
    \label{eq:hste_riemannian}
    \frac{\partial L_{\text{rec}}}{\partial z_e}
    = \frac{1}{\lambda^c_{\hat z}\,\lambda^c_{z_e}}\,\operatorname{gyr}[z_e,-\hat z]\,g_{\hat z} .
\end{equation}
The reconstruction signal reaches the encoder as a single, depth-independent gradient copy rather than through the leaking per-stage cascade, and the commitment terms are filtered identically: stage $i$'s commitment loss pulls $r_{i-1}$ toward $q_i$, but for $i>1$ the residual lies behind the stop-gradient, so only the coarsest $i=1$ term survives. Codebooks remain optimizable through their per-stage codebook loss, and at zero curvature the formulation recovers standard Euclidean residual vector quantization up to a rescaling (Appendix~\ref{app:curvature}). Figure~\ref{fig:ghrq_scheme} summarizes the complete forward pass.

\begin{figure}[t]
    \centering
    \includegraphics[width=0.7\textwidth]{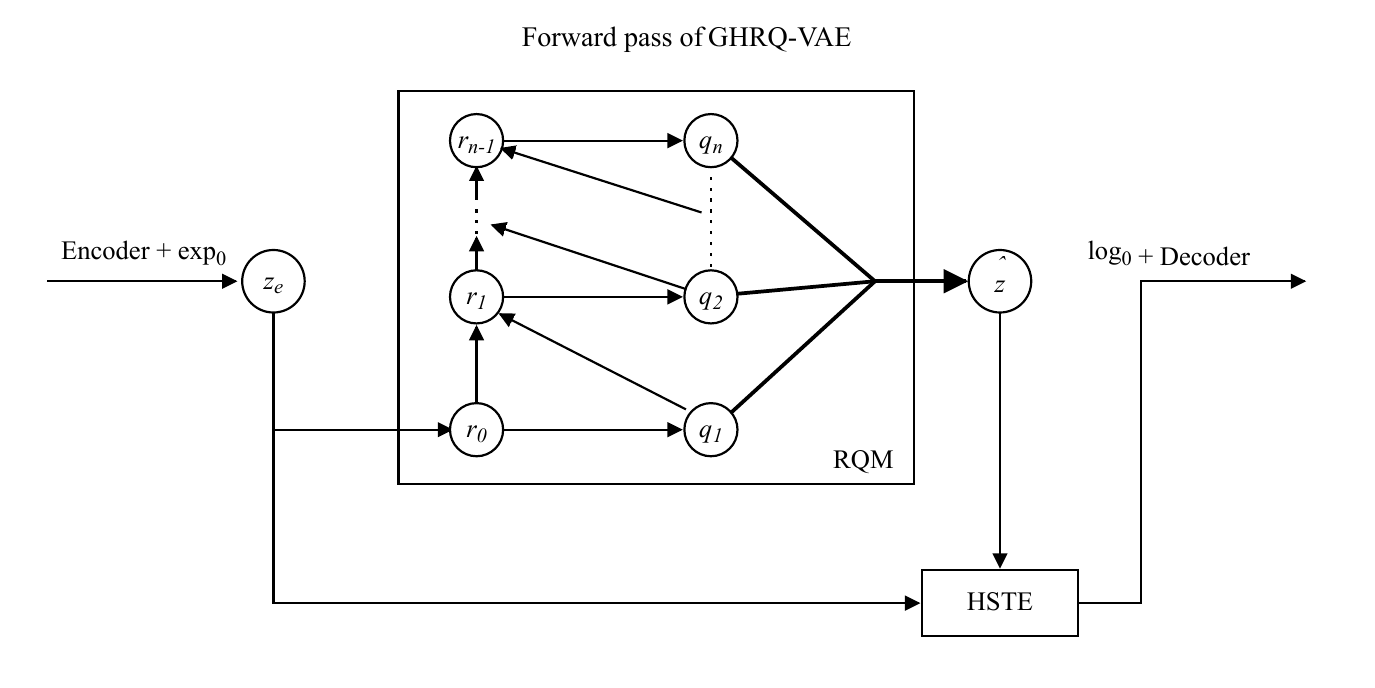}
    \caption{Forward pass of GHRQ. The encoder output $z_e$ is residual-quantized into coarse-to-fine codes $q_1,\dots,q_{N}$, recombined into $\hat{z}$. The block-level gradient routing of \S\ref{sec:blockste} sends a discounted-HSTE gradient from $\hat z$ directly back to $z_e$, bypassing the intermediate codes and residuals. All backward passes through the Residual Quantization Module (RQM) are blocked, except the path $\hat{z} \to q_1\to z_e$, left intact to preserve the coarsest commitment signal as in the Euclidean case. In the naive case, gradients would instead flow and accumulate through all residuals and quantizations of the RQM.}
    \label{fig:ghrq_scheme}
\end{figure}

\section{Experimental Setup}
\label{chap:expsetup}
The methodology of \S\ref{chap:methods} provides a geometry- and architecture-agnostic framework for residual quantization on the Poincar\'e ball. We test the quantizer across four tasks spanning a shallow regime ($N=4$), typical of prior hyperbolic residual quantization studies, and a deep regime ($N=12$), where numerical instabilities become pronounced. Across all experiments, the quantizer geometry and gradient routing are the sole independent variables; the encoder, decoder, optimizer, data pipeline, and evaluation protocols remain fixed within each task. We compare against two baselines with identical configurations. The Euclidean baseline employs standard residual vector quantization ($c=0$) with an identity STE and an additive residual recursion. The naive hyperbolic baseline directly adapts prior hyperbolic residual quantization methods~\cite{hrqvae} ($c=1$): codebooks reside on the Poincar\'e ball and assignments use squared geodesic distance, but the estimator retains the Euclidean identity STE and employs M\"obius addition (Eq.~\ref{eq:prev_residual}). Curvature is set to $c=1$ for hyperbolic models. Encoder and decoder architectures, codebook sizes, optimizers, and training budgets are identical across the three configurations within each task; they are specified in full in Appendix~\ref{app:arch}. The source code will be made publicly available.

\paragraph{Tasks and datasets.}
\label{sec:tasks}
(i) \emph{WordNet hypernymy prediction}~\cite{nickel2017poincare} embeds the $82{,}115$ noun synsets of the WordNet hierarchy in the shallow regime ($N=4$), training the encoder with a contrastive InfoNCE objective ($50$ negatives) on the \emph{closure} split, i.e.\ the transitive closure of the hypernymy DAG, in which a pair $(u,v)$ is positive whenever $v$ is any ancestor of $u$ rather than only its direct hypernym. (ii) \emph{Generative sequential recommendation} follows the semantic-ID paradigm~\cite{tiger}, mapping items of the Amazon Reviews Beauty corpus~\cite{amazonbeauty} (leave-one-out protocol) to discrete codes ($N=4$) over frozen MPNet~\cite{mpnet} sentence embeddings, from which a downstream sequence model generates semantic IDs autoregressively. (iii) \emph{Image reconstruction and generation} uses MNIST~\cite{mnist} and CIFAR-100~\cite{cifar} (whose $100$ fine classes form $20$ coarse superclasses, used only for unsupervised taxonomy evaluation), quantizing a convolutional VQ-VAE~\cite{vqvae} tokenizer ($N=4$) over which an RQ-Transformer prior~\cite{lee2022rqvae} is trained to draw $10{,}000$ samples. (iv) \emph{Neural audio coding} employs a SoundStream-style neural codec~\cite{soundstream,encodec,academicodec} on LibriTTS \texttt{train-clean-100} at $24$\,kHz, the deepest stack ($N=12$). Due to the divergence of conformal factors near the boundary, the hyperbolic codec necessitates explicit encoder stabilization, specifically an auto-calibrated regularizer and quantizer-depth dropout~\cite{soundstream} (Appendix~\ref{app:arch_audio}).

\paragraph{Metrics.}
\label{sec:metrics}
Each configuration is evaluated on task-specific performance and, where applicable, on the structural hierarchy of the learned latent space. WordNet is scored by Recall@10 on the closure split, code-tuple uniqueness, and intra-cluster semantic coherence (WordNet path and Wu--Palmer tree similarity~\cite{wupalmer}, GloVe cosine similarity~\cite{glove}); recommendation by Recall@5/10, NDCG@5/10~\cite{ndcg} and the pre-deduplication uniqueness ratio; images by reconstruction MSE, FID and IS~\cite{inceptionscore}, and unsupervised CIFAR-100 superclass hierarchy recovery via Adjusted Rand Index (ARI)~\cite{ari}, Normalized Mutual Information (NMI), and dendrogram purity; audio by reconstruction loss and perceptual rate-distortion (PESQ~\cite{pesq}, SI-SDR~\cite{sisdr}) against entropy-estimated bitrates.

\section{Results}
\label{chap:experiments}
We evaluate the three quantizer configurations on the four tasks, separating three questions: whether hyperbolic residual quantization improves hierarchical organization, whether our geometric corrections improve stability over naive baselines, and whether these benefits translate into compression. Prior work~\cite{hrqvae} shows hyperbolic residual quantization improves WordNet (\S\ref{sec:exp_nlp}) and sequential recommendation (\S\ref{sec:exp_rec}); on these we assess whether GHRQ-VAE outperforms the naive baseline. We additionally present the first working application of hyperbolic residual quantization to a convolutional image tokenizer (\S\ref{sec:exp_image}) and a deep ($N=12$) neural audio codec (\S\ref{sec:exp_audio}).

\subsection{WordNet Hypernymy Prediction}
\label{sec:exp_nlp}
\begin{table}[t]
\centering
\caption{WordNet hypernymy reconstruction (Recall@10, closure split) across varying encoder dimensions ($d \in \{8, 16\}$) and per-stage codebook sizes ($b \in \{64, 128\}$), with $N=4$.}
\label{tab:nlp_grid}
\begin{tabular}{lcccc}
\hline
Configuration & $d8/b64$ & $d8/b128$ & $d16/b64$ & $d16/b128$ \\
\hline
Euclidean & $75.8$ & $75.4$ & $76.7$ & $75.9$ \\
Naive hyperbolic & $\mathbf{88.2}$ & $62.3$ & $\mathbf{86.9}$ & $\mathbf{83.4}$ \\
GHRQ-VAE (ours) & $81.9$ & $\mathbf{78.7}$ & $83.8$ & $81.9$ \\
\hline
\end{tabular}
\end{table}

Table~\ref{tab:nlp_grid} reports the Recall@10 for hypernymy reconstruction across varying model capacities. While hyperbolic formulations consistently outperform the Euclidean baseline, the naive hyperbolic baseline is unstable. Although it achieves peak recall in specific configurations, it collapses significantly in others ($62.3\%$ at $d8/b128$). Furthermore, its performance may be inflated by lower codebook usage: Table~\ref{tab:nlp_repr} shows that the naive approach yields a low uniqueness ratio ($0.888$), assigning identical code sequences to disparate concepts, thereby reducing the target space for the reconstructor. Conversely, GHRQ-VAE provides a stable optimization profile and maintains high codebook utilization (uniqueness $0.967$). Crucially, as shown in Table~\ref{tab:nlp_repr}, GHRQ-VAE clusters the vocabulary into more semantically coherent groups, consistently outperforming the naive approach across multiple similarity metrics.

\begin{table}[t]
\centering
\caption{WordNet latent representation quality evaluated at the standard capacity setting ($d16/b128$) across all $82{,}115$ noun synsets. The semantic similarity of concepts mapped to identical codes is evaluated using WordNet Path, Wu--Palmer tree similarity, and GloVe cosine similarity, with higher values indicating coherent clustering.}
\label{tab:nlp_repr}
\begin{tabular}{lcccc}
\hline
Configuration & uniq.\ ratio & path sim & wup sim & GloVe sim \\
\hline
Naive hyperbolic & $0.888$ & $0.078$ & $0.242$ & $0.097$ \\
GHRQ-VAE (ours) & $\mathbf{0.967}$ & $\mathbf{0.088}$ & $\mathbf{0.262}$ & $\mathbf{0.162}$ \\
\hline
\end{tabular}
\end{table}

\subsection{Generative Sequential Recommendation}
\label{sec:exp_rec}
The recommendation task evaluates whether the improved hierarchical code space translates to downstream seq2seq generative recommendation. As shown in Table~\ref{tab:rec_beauty}, both hyperbolic methods outperform the Euclidean baseline. GHRQ-VAE achieves the highest performance on the majority of ranking metrics, as well as the highest codebook usage. 

\begin{table}[h]
\centering
\caption{Downstream recommendation performance on the Amazon Beauty dataset. We report the test-set ranking metrics and the pre-deduplication uniqueness ratio (a proxy for codebook health).}
\label{tab:rec_beauty}
\begin{tabular}{lccccc}
\hline
Configuration & uniq.\ ratio & R@5 & NDCG@5 & R@10 & NDCG@10 \\
\hline
Euclidean & $0.960$ & $0.0352$ & $0.0242$ & $0.0530$ & $0.0299$ \\
Naive hyperbolic & $0.870$ & $0.0388$ & $0.0259$ & $\mathbf{0.0606}$ & $0.0329$ \\
GHRQ-VAE (ours) & $\mathbf{0.971}$ & $\mathbf{0.0393}$ & $\mathbf{0.0264}$ & $0.0604$ & $\mathbf{0.0332}$ \\
\hline
\end{tabular}
\end{table}

\subsection{Image Reconstruction and Generation}
\label{sec:exp_image}
The image domain scores the same trained codes across three criteria: compression fidelity (reconstruction MSE), token quality for generative modeling (RQ-Transformer FID/IS), and unsupervised taxonomy discovery. The Euclidean baseline yields the lowest reconstruction error (Table~\ref{tab:img_recon}); both hyperbolic configurations incur a $19$--$36\%$ relative MSE penalty, with GHRQ-VAE marginally behind the naive lift, reflecting that hyperbolic methods are weaker at raw signal reconstruction. When assessing structural organization (Table~\ref{tab:img_hier}), the ranking inverts: GHRQ-VAE increases the Adjusted Rand Index from the Euclidean baseline's $0.048$ to $0.087$ ($+81\%$) and similarly leads in NMI and purity. For datasets possessing a latent taxonomy, the capacity surrendered at the reconstruction stage is recovered here as a more robust global taxonomy. Generation is dataset-dependent (Table~\ref{tab:img_gen}): on MNIST both hyperbolic configurations achieve superior FID, whereas on CIFAR-100 the Euclidean baseline leads, with GHRQ-VAE the strongest hyperbolic alternative.

\begin{table}[t]
\centering
\begin{minipage}{0.48\textwidth}
\centering
\caption{Image reconstruction loss ($\times10^{-3}$). Euclidean achieves the lowest error.}
\label{tab:img_recon}
\begin{tabular}{lcc}
\hline
Configuration & MNIST & CIFAR-100 \\
\hline
Euclidean & $\mathbf{0.477}$ & $\mathbf{1.077}$ \\
Naive hyperbolic & $0.600$ & $1.280$ \\
GHRQ-VAE (ours) & $0.647$ & $1.360$ \\
\hline
\end{tabular}
\end{minipage}
\hfill
\begin{minipage}{0.48\textwidth}
\centering
\caption{Unsupervised CIFAR-100 superclass recovery.}
\label{tab:img_hier}
\begin{tabular}{lccc}
\hline
Configuration & ARI$\uparrow$ & NMI$\uparrow$ & purity$\uparrow$ \\
\hline
Euclidean & $0.048$ & $0.469$ & $0.327$ \\
Naive hyperbolic & $0.055$ & $0.481$ & $0.337$ \\
GHRQ-VAE (ours) & $\mathbf{0.087}$ & $\mathbf{0.515}$ & $\mathbf{0.370}$ \\
\hline
\end{tabular}
\end{minipage}
\end{table}

\begin{table}[t]
\centering
\caption{Image generation performance with an autoregressive RQ-Transformer ($10{,}000$ samples, single seed). We report FID (lower is better) and IS in parentheses (higher is better). }
\label{tab:img_gen}
\begin{tabular}{lcc}
\hline
Configuration & MNIST FID$\downarrow$ (IS$\uparrow$) & CIFAR-100 FID$\downarrow$ (IS$\uparrow$) \\
\hline
Euclidean & $20.36\ (2.069)$ & $\mathbf{94.67}\ (3.86)$ \\
Naive hyperbolic & $\mathbf{15.01}\ (2.105)$ & $101.84\ (3.51)$ \\
GHRQ-VAE (ours) & $16.78\ (2.068)$ & $98.23\ (3.80)$ \\
\hline
\end{tabular}
\end{table}

\subsection{Neural Audio Coding}
\label{sec:exp_audio}
The neural audio coding task uses the deep configuration ($N=12$) to evaluate the scalability of the proposed block-level estimator under extended depth. The hyperbolic codec requires an auto-calibrated encoder-scale control and uniform quantizer-depth dropout to prevent representation collapse across all hyperbolic configurations. Figure~\ref{fig:audio_rd} reports the perceptual rate--distortion behavior across the full depth sweep ($N\in\{1,2,4,8,12\}$), plotting PESQ-wb and SI-SDR against the empirical entropy rate. The Euclidean baseline maintains a superior Pareto frontier at every bitrate, confirming that flat geometry remains preferable for pure signal compression. Among the hyperbolic methods, however, GHRQ-VAE dominates the naive M\"obius lift almost everywhere, attaining higher perceptual quality at matched entropy rate across nearly all operating points and dimensions.

\begin{figure}[h]
\centering
\includegraphics[width=\textwidth]{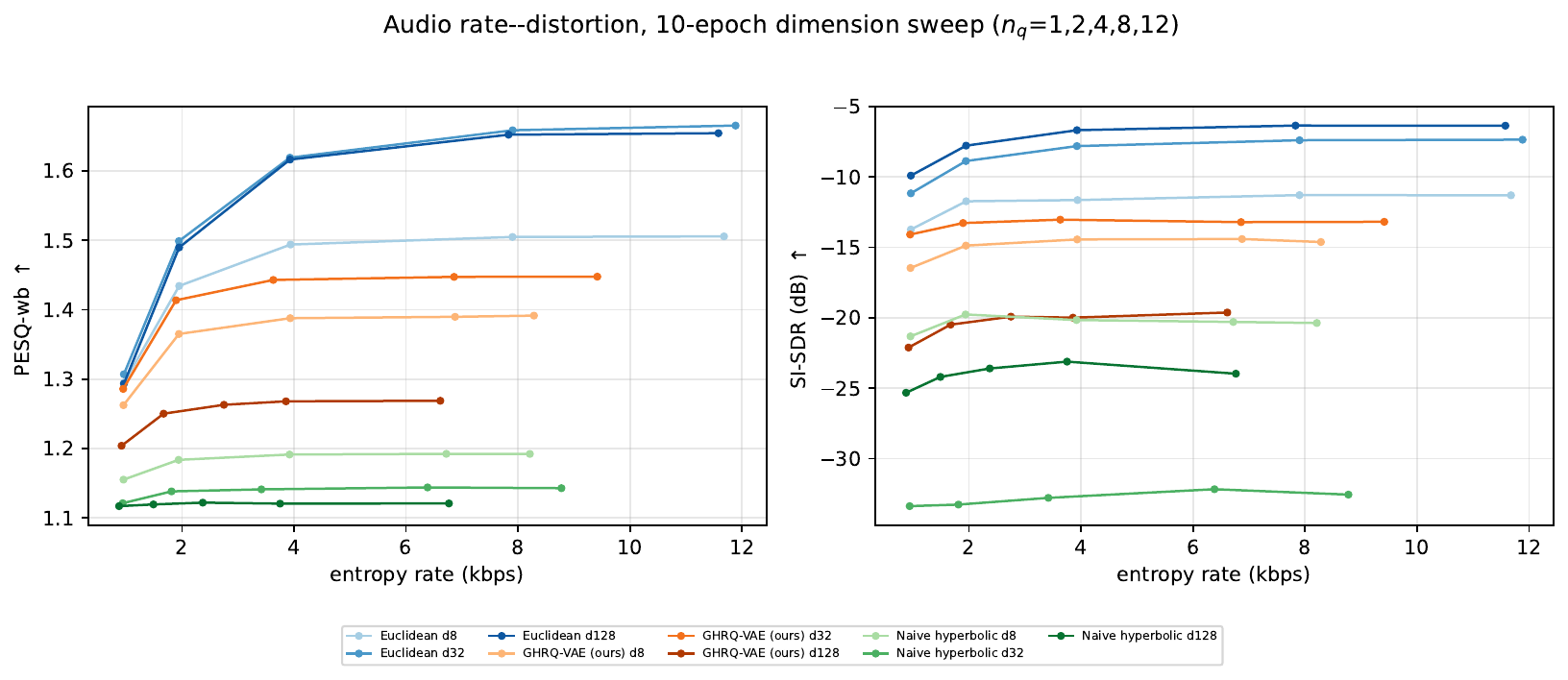}
\caption{Rate--distortion performance across varying quantization depths ($N\in\{1,2,4,8,12\}$), plotted against the empirical entropy rate. Individual markers denote distinct operating points. Color indicates the quantizer variant (blue: Euclidean; orange: GHRQ-VAE; green: naive hyperbolic lift), while shading represents the codebook dimensionality ($d\in\{8,32,128\}$). Left: PESQ-wb. Right: SI-SDR. The curves plateau when subsequent residual stages increase the empirical entropy rate without proportional gains in perceptual quality.}
\label{fig:audio_rd}
\end{figure}

\subsection{Residual Reconstruction Error Across Tasks}
\label{sec:exp_approx}
This section evaluates whether GHRQ-VAE lets the residual cascade accurately reconstruct the encoder output on the Poincar\'e ball, overcoming the gyration drift of the naive baseline (Eq.~\ref{eq:prev_residual}), quantified by the squared hyperbolic distance between the aggregated codes and the encoder output. The results (Table~\ref{tab:approx}) corroborate the predictions of \S\ref{sec:newmethod}: GHRQ-VAE reconstructs the encoder representation with high precision, yielding residual errors near zero.

\begin{table}[t]
\centering
\caption{Mean validation reconstruction error on the Poincar\'e ball, measured as the squared hyperbolic distance between the reconstructed code and the encoder output (audio trained for $10$ epochs). Lower is better; the best result per row is in bold.}
\label{tab:approx}
\begin{tabular}{lcc}
\hline
Task & Naive hyperbolic & GHRQ-VAE (ours) \\
\hline
WordNet hypernymy ($d16/b128$)  & $14.4$    & $\mathbf{12.9}$ \\
Sequential recommendation & $44.79$ & $\mathbf{0.031}$ \\
Image reconstruction (MNIST)     & $0.00049$ & $\mathbf{<\!10^{-5}}$ \\
Image reconstruction (CIFAR-100) & $0.00133$ & $\mathbf{<\!10^{-5}}$ \\
Neural audio coding (mean over $d\in\{8,32,128\}$) & $0.521$ & $\mathbf{0.0044}$ \\
\hline
\end{tabular}
\end{table}

\subsection{Ablation}
\label{ch:ablation}
\label{sec:abl_components}
GHRQ-VAE combines two modifications to the naive baseline: the HRA forward ordering and the block-level d-HSTE gradient. To identify which one drives the reduction in on-ball residual error, we ablate each in turn while holding all other hyperparameters fixed: \emph{HRA only} keeps the forward ordering but reverts the gradient to the identity STE, and \emph{d-HSTE only} applies the corrected gradient over the naive M\"obius aggregation (Eq.~\ref{eq:prev_residual}). Table~\ref{tab:abl_residual} reports the residual error of the two ablations against the full model.

\begin{table}[t]
\centering
\caption{Component ablation on the on-ball residual error ($N=4$), measured as the mean validation squared hyperbolic distance between the recomposed codes and the encoder output; image columns are scaled by $10^{-3}$. The full-model values correspond to those of Table~\ref{tab:approx}. Bold indicates the best result per column.}
\label{tab:abl_residual}
\begin{tabular}{cc|cccc}
\hline
HRA & d-HSTE & WordNet$\downarrow$ & Recommendation$\downarrow$ & MNIST ($\times10^{-3}$)$\downarrow$ & CIFAR ($\times10^{-3}$)$\downarrow$ \\
\hline
\checkmark & $\times$     & $\mathbf{10.3}$ & $5.49$          & $0.01$          & $0.03$ \\
$\times$   & \checkmark   & $30.4$          & $0.12$          & $0.02$          & $0.05$ \\
\checkmark & \checkmark   & $12.9$          & $\mathbf{0.03}$ & $\mathbf{0.00}$ & $\mathbf{0.00}$ \\
\hline
\end{tabular}
\end{table}

HRA guarantees telescoping algebraically, yet when paired with the uncorrected gradient it is the \emph{least} faithful configuration on recommendation. Symmetrically, dropping HRA is harmless on recommendation and images but inflates the WordNet error to $30.4$. The two modifications are thus complementary rather than independent: HRA's telescoping guarantee holds only at the operating radii that d-HSTE trains the encoder to visit, while the corrected gradient recomposes faithfully only when the forward pass is ordered to telescope. Each component supplies the precondition the other requires, and the residual error collapses uniformly only when both are present.

\section{Conclusion and Discussion}
Across four domains, a consistent picture emerges: hyperbolic curvature is most valuable for structure, not compression. Euclidean models retain the edge on raw reconstruction and on CIFAR-100 generation quality, yet the curved latent space organizes hierarchy more naturally, yielding substantial gains in unsupervised CIFAR-100 superclass recovery and WordNet hypernymy reconstruction. Relative to the naive hyperbolic lift, GHRQ-VAE is more robust, less prone to instability and codebook collapse, and its geometric corrections align the residual codes with the encoder point on the manifold. The value of GHRQ-VAE thus lies in stable hierarchical structuring rather than signal fidelity.

This separation also reframes an open question around the naive lift. Despite being geometrically inexact, it remains strong at shallow depth, attaining the best WordNet recall and leading on recommendation R@10. A plausible cause is its leaked per-stage gradient: geometrically inaccurate, yet behaving as a low-variance, on-average-correct directional signal; a rigorous study of how to fold it into a geometrically consistent estimator is one direction for future work.

Closing this gap, and following the field's shift toward diffusion and flow-matching generation \cite{resgen, hypdae, bu2025ggball}, evaluating hyperbolic tokenizers within these continuous paradigms and within large-scale modern tokenizers is a natural next step.

\bibliographystyle{splncs04}
\bibliography{bibentries}

\clearpage
\setcounter{section}{0}
\renewcommand{\thesection}{\Alph{section}}
\renewcommand{\theHsection}{app.\Alph{section}}
\begin{center}
    {\Large\bfseries Supplementary Material\par}
\end{center}
\noindent
This supplementary material contains the two items deferred from the main paper. Appendix~\ref{app:arch} gives the complete architecture, dataset-processing, and optimization details of every experiment, so that all four setups are reproducible from the description alone. Appendix~\ref{app:proofs} collects proofs of claims that the main paper states without derivation, and we show the recovery of Euclidean RQ at zero curvature.

\section{Architecture and Implementation Details}
\label{app:arch}

All experiments were run on a SLURM-managed cluster of NVIDIA A100 and H100 GPUs. Within each task the encoder, decoder, downstream model, data pipeline, and evaluation protocol are held fixed; the quantizer geometry and gradient routing are the only independent variables. Every configuration of a given task shares the number of residual stages $N$, the codebook sizes, and the encoder/decoder architectures, so that differences in the reported metrics are attributable to the quantizer alone.

\subsection{Quantizer Configurations}
\label{app:arch_configs}
\paragraph{Euclidean (baseline).} Standard residual vector quantization at $c=0$, with the identity straight-through estimator and the additive recursion $r_i=r_{i-1}-q_i$, $\hat z=\sum_i q_i$.

\paragraph{Naive hyperbolic.} A direct adaptation of prior hyperbolic residual quantization~\cite{hrqvae} at $c=1$. Codebooks are parameters on the Poincar\'e ball and assignments minimize the squared geodesic distance, i.e.\ $q(x)=c_k$ with
\begin{equation*}
    k=\operatorname*{argmin}_{j} d^2_{\mathbb{D}_c}(x,c_j) ,
\end{equation*}
but the estimator keeps the Euclidean identity STE and the left-associated M\"obius aggregation of Eq.~\ref{eq:prev_residual}.

\paragraph{GHRQ-VAE (ours).} Identical to the naive hyperbolic configuration except for the two repairs of \S\ref{chap:methods}: the HRA forward convention (Eq.~\ref{eq:hrq_residual}) and block-level routing of a single d-HSTE step (Eq.~\ref{eq:hste_riemannian}), with the numerically stable gyration of Eq.~\ref{eq:gyr_stable}.

\subsection{WordNet Hypernymy Prediction}
\label{app:arch_wordnet}
The task embeds the WordNet noun hierarchy, comprising $82{,}115$ noun synsets and their hypernymy edges as extracted with NLTK. A fully-connected embedding network maps each synset to a $16$-dimensional vector, which is projected onto the manifold and quantized with $N=4$ stages and $128$ codes per stage in the standard setting; the capacity grid of Table~\ref{tab:nlp_grid} sweeps encoder dimension $d\in\{8,16\}$ and per-stage codebook size $b\in\{64,128\}$. The encoder is trained with a contrastive InfoNCE objective that contrasts each positive hypernymy edge against $50$ sampled negatives; negatives are drawn from outside the transitive closure so that true ancestors are never sampled as negatives. Evaluation uses the \emph{closure} split, which requires composing transitive relations to reconstruct held-out edges. Recall@10 is measured with an autoregressive seq2seq model that predicts hypernym code tuples under beam search, trained for $10$ epochs on the frozen representations. For calibration, a no-model graph-composition baseline reaches roughly $80\%$ Recall@10 on this split and a global-popularity baseline roughly $41\%$.

\subsection{Generative Sequential Recommendation}
\label{app:arch_rec}
We follow the semantic-ID paradigm~\cite{tiger}: each item is mapped to a tuple of discrete codes by residual quantization, and a sequence model predicts the codes of the next item from the user's interaction history. The data is the Beauty category of the Amazon Reviews 2014 corpus~\cite{amazonbeauty} under the standard leave-one-out protocol. Items are first encoded into $768$-dimensional sentence embeddings with a pretrained MPNet sentence-transformer~\cite{mpnet}; these embeddings are frozen and serve as the input to the quantizer. The quantizer is an RQ-VAE~\cite{lee2022rqvae,hrqvae} with a $768\!\to\!512\!\to\!32$ MLP encoder and a symmetric decoder, $N=4$ stages and $128$ codes per stage. The downstream recommender is an encoder--decoder Transformer of model dimension $384$ with $6$ layers, $6$ attention heads, feed-forward dimension $1024$ and dropout $0.1$, which generates semantic IDs autoregressively by beam search with $50$ beams over a history length of $20$. As is standard, a uniqueness tie-break token is appended to disambiguate items that map to identical code tuples.

\subsection{Image Reconstruction and Generation}
\label{app:arch_image}
The image experiments use MNIST~\cite{mnist} ($28\times28$ grayscale; $60{,}000$ training and $10{,}000$ test images) and CIFAR-100~\cite{cifar} ($32\times32$ RGB; $50{,}000$ training and $10{,}000$ test images). CIFAR-100 supplies a two-level label hierarchy, $100$ fine classes grouped into $20$ coarse superclasses, which is used only to score unsupervised taxonomy recovery and never as a training signal. Pixel values are normalized to $[-1,1]$ and no data augmentation is applied, so that the quantizer's effect is isolated.

The tokenizer is a convolutional VQ-VAE~\cite{vqvae}. The encoder is a stack of strided $2$-D convolutions with channel widths $\text{in}\to32\to64\to128\to D$, giving a spatial downsampling factor of four, followed by the residual quantizer; the decoder mirrors it with transposed convolutions. We use $N=4$ stages with latent dimension $D=8$ and $128$ codes per stage on MNIST, and $D=16$ with $512$ codes per stage on CIFAR-100. For generation, an RQ-Transformer prior~\cite{lee2022rqvae} consisting of a spatial and a depth transformer, each with $4$ layers, $8$ heads and model dimension $256$, is trained over the frozen quantizer, and $10{,}000$ samples are drawn from it for FID and IS.

\subsection{Neural Audio Coding}
\label{app:arch_audio}
The codec is a SoundStream-style neural audio codec~\cite{soundstream,encodec} built on the AcademiCodec implementation~\cite{academicodec}, using a SEANet convolutional encoder--decoder paired with a $12$-stage residual quantizer with $1024$ codes per stage. It is trained on the \texttt{train-clean-100} subset of LibriTTS at a $24$\,kHz sample rate. This is the deepest evaluated stack and drives the intermediate residuals close to the boundary of the ball, which is precisely the regime in which the leaked per-stage gradient of Appendix~\ref{app:leak} diverges. Training uses an adversarial objective with multi-scale STFT, multi-period and multi-scale waveform discriminators, together with reconstruction and feature-matching losses; the adversarial terms are switched on after $500$ steps. All configurations are trained under a $10$-epoch budget, and the rate--distortion sweep of Fig.~\ref{fig:audio_rd} additionally varies $N\in\{1,2,4,8,12\}$ and $d\in\{8,32,128\}$.

\paragraph{Encoder-scale control.} High-dimensional encoder outputs have tangent norm of order $\sqrt{d}$ and are therefore mapped essentially onto the boundary by $\exp^c_0$, which saturates the geodesic distance and produces vanishing gradients. We counter this with an auto-calibrated global multiplier applied to the encoder tangent vectors, tuned so that the median residual radius is $0.5$. The multiplier is maintained by an exponential moving average, $s\leftarrow 0.99\,s+0.01\,s_{\text{batch}}$, which prevents deep residuals from drifting toward the boundary over training. Uniform quantizer-depth dropout~\cite{soundstream} is applied alongside it; without both mechanisms all hyperbolic configurations collapse.

\subsection{Optimization}
\label{app:arch_hparams}
Base parameters are optimized with AdamW and manifold parameters (the codebooks) with Riemannian Adam~\cite{geoopt}, except on WordNet, where the encoder is trained with Riemannian SGD. Curvature is fixed at $c=1$ for all hyperbolic models. Table~\ref{tab:supp_hparams} lists the per-task budgets, learning rates and loss weights.

\begin{table}[h]
\centering
\caption{Per-task optimization settings. ``lr'' is the base learning rate and ``codebook lr'' the learning rate of the manifold parameters; $\beta$ is the commitment weight. Downstream models (seq2seq reconstructor, recommender, RQ-Transformer prior) are trained on frozen quantizers.}
\label{tab:supp_hparams}
\begin{tabular}{lccccc}
\hline
Task & optimizer & epochs & lr & codebook lr & $\beta$ \\
\hline
WordNet encoder      & Riem.\ SGD & $50$   & $1.0$              & $1.0$      & $1.0$  \\
\quad + seq2seq recall model & AdamW & $10$ & $10^{-4}$ & --- & --- \\
Recommendation RQ-VAE & AdamW & $5000$ & $10^{-4}$          & $10^{-4}$  & $0.01$ \\
\quad + recommender  & AdamW & $100$  & $10^{-4}$          & ---        & ---    \\
Image VQ-VAE         & AdamW & $50$   & $3\times10^{-4}$   & $10^{-4}$  & $0.25$ \\
Audio codec          & AdamW & $10$   & $3\times10^{-4}$   & $10^{-4}$  & $0.25$ \\
\hline
\end{tabular}
\end{table}

\noindent
The recommendation RQ-VAE additionally weights its reconstruction term by $1000$ relative to the commitment term, which is the setting under which all three quantizer configurations converge. 

\subsection{Evaluation Protocol}
\label{app:arch_metrics}
WordNet path similarity is the inverse graph distance in the taxonomy and Wu--Palmer similarity~\cite{wupalmer} derives relatedness from the depth of the lowest common ancestor; both are averaged over pairs of synsets that receive identical code tuples, alongside the cosine similarity of their GloVe embeddings~\cite{glove}. Recommendation reports Recall@$5$/$10$ and NDCG@$5$/$10$~\cite{ndcg} on held-out items, with the pre-deduplication uniqueness ratio of generated semantic IDs as a diagnostic of codebook utilization. Image reconstruction is scored by the best validation MSE and generation by FID and IS~\cite{inceptionscore} over $10{,}000$ samples. CIFAR-100 hierarchy recovery is computed by agglomeratively clustering the mean code embeddings of the $100$ fine classes into $20$ groups and comparing them against the ground-truth superclasses via the Adjusted Rand Index~\cite{ari}, Normalized Mutual Information and purity. Audio reports the best validation reconstruction loss and perceptual rate--distortion in PESQ~\cite{pesq} and SI-SDR~\cite{sisdr} against entropy-estimated bitrates.

\section{Proofs}
\label{app:proofs}

Throughout, $\mathbb{D}^d_c=\{x\in\mathbb{R}^d: c\|x\|^2<1\}$ carries the M\"obius addition of Eq.~\ref{eq:bg_mobius_add}, and $\ominus x := -x$. We write $\lambda^c_x=2/(1-c\|x\|^2)$ for the conformal factor and $\operatorname{gyr}[u,v]$ for the gyration of Eq.~\ref{eq:bg_gyration}. We use two standard facts about the M\"obius gyrogroup $(\mathbb{D}^d_c,\oplus_c)$~\cite{ungar2009gyrovector,ganea2018hyperbolic}.

\begin{itemize}
    \item[(G1)] \emph{Left gyroassociativity.} $a\oplus_c(b\oplus_c w)=(a\oplus_c b)\oplus_c\operatorname{gyr}[a,b]\,w$, together with $a\oplus_c(-a)=0$, $\operatorname{gyr}[a,-a]=\mathrm{Id}$, and $\operatorname{gyr}[a,0]=\mathrm{Id}$.
    \item[(G2)] \emph{Gyrations are rotations.} For every $u,v\in\mathbb{D}^d_c$ the map $\operatorname{gyr}[u,v]:\mathbb{R}^d\to\mathbb{R}^d$ is linear and orthogonal, so $\|\operatorname{gyr}[u,v]\,w\|=\|w\|$ for all $w$; it acts as the identity on the orthogonal complement of $\operatorname{span}\{u,v\}$.
\end{itemize}

\subsection{The Left-Cancellation Law}
\label{app:leftcancel}

\begin{proposition}
\label{prop:leftcancel}
For all $a,b\in\mathbb{D}^d_c$, $\;a\oplus_c\big((-a)\oplus_c b\big)=b$, which is Eq.~\ref{eq:left_cancel}.
\end{proposition}

\begin{proof}
Apply left gyroassociativity (G1) with the middle argument $-a$:
\begin{equation*}
    a\oplus_c\big((-a)\oplus_c b\big)
    = \big(a\oplus_c(-a)\big)\oplus_c\operatorname{gyr}[a,-a]\,b
    = 0\oplus_c\mathrm{Id}\,b
    = b,
\end{equation*}
using $a\oplus_c(-a)=0$, $\operatorname{gyr}[a,-a]=\mathrm{Id}$, and the fact that $0$ is the identity element of $\oplus_c$.\qed
\end{proof}

The asymmetry that HRA exploits is that the corresponding \emph{right}-hand statement is false: $\big(b\oplus_c(-a)\big)\oplus_c a\neq b$ in general, because $\oplus_c$ is not associative. A residual update that subtracts on the left is therefore inverted by an aggregation that adds on the left, and by no other.

\subsection{The HRA Residual Mismatch is a Pure Rotation}
\label{app:hra_rotation}

Let $\hat z_i:=q_1\oplus_c(q_2\oplus_c(\cdots\oplus_c q_i))$ be the aggregate of the first $i$ codes and let
\begin{equation}
    \label{eq:app_true_residual}
    R_i^{\text{true}}:=(-\hat z_i)\oplus_c z_e,
    \qquad\text{so that}\qquad
    \hat z_i\oplus_c R_i^{\text{true}}=z_e
\end{equation}
be the true residual, the exact part of $z_e$ not yet captured by the first $i$ codes. This is the quantity each codebook is meant to be fitted to, and it is in general distinct from the tracked residual $r_i$.

\begin{proposition}
\label{prop:rotation}
Under the HRA recursion, for every $i=1,\dots,N$,
\begin{equation}
    \label{eq:app_gamma}
    R_i^{\text{true}}=\Gamma_i\,r_i,
    \qquad
    \Gamma_i=\prod_{k=1}^{i-1}\operatorname{gyr}\!\big[q_k,\,u_{k+1}\big],
    \qquad
    u_k:=q_k\oplus_c\big(q_{k+1}\oplus_c(\cdots\oplus_c q_i)\big),
\end{equation}
which is Eq.~\ref{eq:hrq_gyration}; here $u_i=q_i$, $u_1=\hat z_i$, the product is taken in increasing $k$ from left to right, and it is empty (hence $\mathrm{Id}$) when $i=1$. Since $\Gamma_i$ is a composition of gyrations it is orthogonal, and therefore
\begin{equation}
    \label{eq:app_norm}
    \big\|R_i^{\text{true}}\big\|=\big\|r_i\big\| .
\end{equation}
\end{proposition}

\begin{proof}
We use the gyrotranslation identity~\cite{ungar2009gyrovector}
\begin{equation}
    \label{eq:app_left_cancel_gyr}
    -(a\oplus_c b)\oplus_c(a\oplus_c w)=\operatorname{gyr}[a,b]\big((-b)\oplus_c w\big).
\end{equation}
Alongside the tail aggregates $u_k$ of Eq.~\ref{eq:app_gamma}, set $w_k:=q_k\oplus_c\big(q_{k+1}\oplus_c(\cdots\oplus_c(q_i\oplus_c r_i))\big)$ for $k=1,\dots,i$, so that $u_k=q_k\oplus_c u_{k+1}$ and $w_k=q_k\oplus_c w_{k+1}$. Proposition~\ref{prop:leftcancel} with $a=q_k$ and $b=r_{k-1}$ gives the telescoping identity $q_k\oplus_c r_k=r_{k-1}$, so $w_i=r_{i-1}$ and, descending, $w_k=r_{k-1}$; in particular $w_1=r_0=z_e$, while $u_1=\hat z_i$.

Applying Eq.~\ref{eq:app_left_cancel_gyr} with $a=q_k$, $b=u_{k+1}$ and $w=w_{k+1}$ gives, for $k=1,\dots,i-1$,
\begin{equation*}
    (-u_k)\oplus_c w_k
    = -(q_k\oplus_c u_{k+1})\oplus_c(q_k\oplus_c w_{k+1})
    = \operatorname{gyr}[q_k,u_{k+1}]\big((-u_{k+1})\oplus_c w_{k+1}\big),
\end{equation*}
while at $k=i$ the left-cancellation law gives $(-u_i)\oplus_c w_i=(-q_i)\oplus_c(q_i\oplus_c r_i)=r_i$. Composing these $i-1$ steps,
\begin{equation*}
    R_i^{\text{true}}=(-\hat z_i)\oplus_c z_e=(-u_1)\oplus_c w_1
    =\Big(\prod_{k=1}^{i-1}\operatorname{gyr}[q_k,u_{k+1}]\Big)\,r_i=\Gamma_i\,r_i,
\end{equation*}
which is Eq.~\ref{eq:app_gamma}; for $i=1$ the product is empty and $R_1^{\text{true}}=(-q_1)\oplus_c z_e=r_1$ directly. Each factor of $\Gamma_i$ is orthogonal by (G2) and a product of orthogonal maps is orthogonal, so $\|R_i^{\text{true}}\|=\|\Gamma_i r_i\|=\|r_i\|$.\qed
\end{proof}

Proposition~\ref{prop:rotation} is the precise sense in which HRA repairs the forward pass: the mismatch between the residual the quantizer tracks and the residual it ought to track is a rotation about the origin, which contributes \emph{zero} magnitude error at every depth. The coarse-to-fine magnitude decomposition on which residual quantization rests therefore remains faithful, in contrast with the naive convention, whose mismatch is a drift that corrupts $\|r_i\|$ and accumulates with $i$.

\subsection{The Residual Gradient Leaks on the Ball}
\label{app:leak}

In Euclidean residual quantization the identity STE has Jacobian $\partial q_i/\partial r_{i-1}=I$ and the residual update is additive, so
\begin{equation}
    \label{eq:app_firewall}
    \frac{\partial r_i}{\partial r_{i-1}}=\frac{\partial(r_{i-1}-q_i)}{\partial r_{i-1}}=I-I=0 .
\end{equation}
Thanks to this exact cancellation the residual branch transmits nothing, so the encoder receives exactly one copy of the decoder gradient (through the shortest path, whose empty product of residual Jacobians is $I$) together with the lone $i=1$ commitment term, independently of the depth $N$. Proposition~\ref{prop:leak} shows the cancellation fails on the ball.

\begin{proposition}
\label{prop:leak}
Let $r:=r_{i-1}$, $q:=q_i$ and let $J_i:=\partial q_i/\partial r_{i-1}$ be the straight-through Jacobian. For $r_i=r_{i-1}\oplus_c(-q_i)$,
\begin{equation}
    \label{eq:app_leak}
    A_i:=\frac{\partial r_i}{\partial r_{i-1}}
    = D_1\!\oplus_c(r,-q)\;-\;D_2\!\oplus_c(r,-q)\,J_i ,
\end{equation}
where $D_1\!\oplus_c$ and $D_2\!\oplus_c$ are the Jacobians of M\"obius addition in its first and second argument. At $c=0$ one has $D_1=D_2=I$, so the straight-through convention $J_i=I$ gives $A_i=0$. For $c>0$ and $J_i=I$, however, whenever $\operatorname{span}\{r,q\}^{\perp}\neq\{0\}$ --- in particular for every $d\ge3$ --- one has $A_i=0$ if and only if $q_i=r_{i-1}$: it leaks at every stage at which the quantization is not exact.
\end{proposition}

Unrolling the recursion with $r_0=z_e$, the chain rule now routes \emph{both} signals of Eq.~\ref{eq:rqvae_loss} to the encoder through these leak products,
\begin{equation}
    \label{eq:app_stack}
    \nabla_{r_0}\mathcal{L}
    = \underbrace{\sum_{i=1}^{N}\Big(\textstyle\prod_{j=1}^{i-1}A_j^{\top}\Big)J_i^{\top}\,\nabla_{q_i}L_{\text{rec}}}_{\text{reconstruction, via the codes}}
    \;+\;
    \underbrace{\sum_{i=1}^{N}\Big(\textstyle\prod_{j=1}^{i-1}A_j^{\top}\Big)\nabla_{r_{i-1}}L^{\text{commit}}_i}_{\text{commitment, via the residuals}},
\end{equation}
with the empty product at $i=1$ equal to $I$. Every stage contributes, so instead of the two clean copies of the Euclidean case the encoder collects a superposition of $N$ reconstruction and $N$ commitment terms, each filtered through a different-length product of leak matrices. Proposition~\ref{prop:leak} shows the individual factors do not vanish, and the explicit form obtained in its proof (Eq.~\ref{eq:app_leak_perp}) makes their growth explicit: as the residuals approach the boundary of the ball, $\|A_j\|$ diverges, so the superposition amplifies with both depth and radius. This is why the forward repair alone is insufficient and the backward pass must be repaired independently: applying stop-gradients to $r_i$ for $i\ge1$ sets every $A_j$ path to zero by construction, and the single d-HSTE hop of Eq.~\ref{eq:hste_riemannian} reinstates the one depth-independent copy of the reconstruction gradient that Eq.~\ref{eq:app_firewall} used to guarantee.

It remains to prove Proposition~\ref{prop:leak}; the computation below quantifies the leak but is not needed to follow the argument above.

\begin{proof}[of Proposition~\ref{prop:leak}]
Equation~\ref{eq:app_leak} is the chain rule applied to the two arguments of $\oplus_c$, the second contributing through $q_i=q(r_{i-1})$ with a minus sign; at $c=0$ M\"obius addition is ordinary addition, so $D_1=D_2=I$ and the two terms cancel, which is Eq.~\ref{eq:app_firewall}.

Let $c>0$, put $\delta:=q-r$ and $\gamma:=1-2c\langle r,q\rangle+c^2\|r\|^2\|q\|^2$, and set $x:=r$, $y:=-q$, so that $x\oplus_c y=(\alpha x+\beta y)/\gamma$ with $\alpha:=1+2c\langle x,y\rangle+c\|y\|^2$ and $\beta:=1-c\|x\|^2$; Cauchy--Schwarz gives $\gamma\ge(1-c\|x\|\|y\|)^2>0$. Differentiating the quotient in $x$ and applying the result to a $w\perp\operatorname{span}\{x,y\}$, every term carrying a factor $\langle x,w\rangle$ or $\langle y,w\rangle$ drops out and only $D_1\!\oplus_c(x,y)\,w=(\alpha/\gamma)\,w$ survives. For the second argument, $y\mapsto x\oplus_c y$ is an isometry of $\mathbb{D}^d_c$, whose differential is $D_2\!\oplus_c(x,y)=(\lambda^c_y/\lambda^c_{x\oplus_c y})\operatorname{gyr}[x,y]$; the identity $1-c\|x\oplus_c y\|^2=(1-c\|x\|^2)(1-c\|y\|^2)/\gamma$ reduces the prefactor to $\beta/\gamma$, and $\operatorname{gyr}[x,y]$ fixes $\operatorname{span}\{x,y\}^{\perp}$ pointwise by (G2), so $D_2\!\oplus_c(x,y)\,w=(\beta/\gamma)\,w$ there. With $J_i=I$ this gives $A_i w=(\alpha-\beta)\gamma^{-1}w$, and since
\begin{equation*}
    \alpha-\beta=c\big(\|q\|^2+\|r\|^2-2\langle r,q\rangle\big)=c\|\delta\|^2 ,
\end{equation*}
the leak acts on $\operatorname{span}\{r,q\}^{\perp}$ as the strictly positive scalar
\begin{equation}
    \label{eq:app_leak_perp}
    A_i\,w=\frac{c\,\|\delta\|^2}{\gamma}\,w ,
    \qquad w\perp\operatorname{span}\{r,q\} .
\end{equation}
In the collinear case $r=\rho e$ and $q=\kappa e$ ($\|e\|=1$), $\oplus_c$ restricts to the one-dimensional law $\rho\oplus_c(-\kappa)=(\rho-\kappa)/(1-c\rho\kappa)$ with $\gamma=(1-c\rho\kappa)^2$, whose two partial derivatives are $(1-c\kappa^2)/\gamma$ and $(1-c\rho^2)/\gamma$; their difference gives the action on the radial direction,
\begin{equation}
    \label{eq:app_leak_1d}
    A_i\,e=\frac{c\,(\rho^2-\kappa^2)}{(1-c\rho\kappa)^2}\,e
    =\frac{2}{(1-c\rho\kappa)^2}\left(\frac{1}{\lambda^c_q}-\frac{1}{\lambda^c_r}\right)e ,
\end{equation}
using $1-c\|x\|^2=2/\lambda^c_x$.

Finally, $c\|\delta\|^2/\gamma=0$ forces $\delta=0$. Conversely, if $q=r$ then $y=-x$, so $\alpha=\beta=1-c\|r\|^2$, $\alpha x+\beta y=0$ and $\operatorname{gyr}[x,-x]=\mathrm{Id}$, whence $D_1=D_2=(1-c\|r\|^2)^{-1}I$ and $A_i=0$.\qed
\end{proof}

The same conclusion holds under the HRA convention: with $r_i=(-q_i)\oplus_c r_{i-1}$ the roles of $D_1$ and $D_2$ are exchanged and the identical computation gives $A_iw=-c\|\delta\|^2\gamma^{-1}w$ on $\operatorname{span}\{r_{i-1},q_i\}^{\perp}$, of the same magnitude.

\subsection{Exactness of the Numerically Stable Gyration}
\label{app:gyration}

The backward pass evaluates $\operatorname{gyr}[z_e,-q]$ when its base points nearly coincide, since $q$ quantizes $z_e$ (and at the block level $\hat z\approx z_e$). Writing the closed form as $\operatorname{gyr}[A,B]\,v=v+2(aA+bB)/d$ with $A=z_e$ and $B=-q$,
\begin{align}
    \label{eq:app_gyr_raw}
    a &= -c^2\langle A,v\rangle\|B\|^2+c\langle B,v\rangle+2c^2\langle A,B\rangle\langle B,v\rangle,\notag\\
    b &= -c^2\langle B,v\rangle\|A\|^2-c\langle A,v\rangle,\notag\\
    d &= 1+2c\langle A,B\rangle+c^2\|A\|^2\|B\|^2 .
\end{align}
With $\delta:=q-z_e$ small one has $A\approx-B$, and the evaluation cancels catastrophically twice: $d$ collapses to $\approx(1-c\|z_e\|^2)^2$, computed as a difference of $O(1)$ quantities that itself vanishes at the boundary, while $aA+bB=a\,z_e-b\,q$ subtracts two nearly equal vectors because $a\approx b$. Proposition~\ref{prop:gyrstable} states that the reformulation used in Eq.~\ref{eq:gyr_stable} is not an approximation but an algebraic identity.

\begin{proposition}
\label{prop:gyrstable}
With $A=z_e$, $B=-q$ and $\delta=q-z_e$, the quantities of Eq.~\ref{eq:app_gyr_raw} satisfy, exactly,
\begin{align}
    \label{eq:app_gyr_stable}
    d   &= (1-c\|z_e\|^2)^2-2c\,(1-c\|z_e\|^2)\,\langle z_e,\delta\rangle+c^2\|z_e\|^2\|\delta\|^2,\notag\\
    a-b &= -c\,(1-c\|z_e\|^2)\,\langle\delta,v\rangle-c^2\langle z_e,v\rangle\,\|\delta\|^2+2c^2\langle z_e,\delta\rangle\,\langle\delta,v\rangle,
\end{align}
and the numerator decomposes as $aA+bB=\tfrac{a-b}{2}(q+z_e)-\tfrac{a+b}{2}\,\delta$.
\end{proposition}

\begin{proof}
Write $z:=z_e$, so $q=z+\delta$, $\langle A,B\rangle=-\|z\|^2-\langle z,\delta\rangle$, $\|B\|^2=\|z\|^2+2\langle z,\delta\rangle+\|\delta\|^2$ and $\|A\|^2=\|z\|^2$. Substituting into $d$,
\begin{align*}
    d &= 1-2c\|z\|^2-2c\langle z,\delta\rangle+c^2\|z\|^2\big(\|z\|^2+2\langle z,\delta\rangle+\|\delta\|^2\big)\\
      &= \big(1-2c\|z\|^2+c^2\|z\|^4\big)-2c\langle z,\delta\rangle\big(1-c\|z\|^2\big)+c^2\|z\|^2\|\delta\|^2,
\end{align*}
which is the first line of Eq.~\ref{eq:app_gyr_stable} since the first bracket is $(1-c\|z\|^2)^2$. Every term after the first is $O(\|\delta\|)$, so no $O(1)$ cancellation occurs.

For the second line, substitute $\langle A,v\rangle=\langle z,v\rangle$ and $\langle B,v\rangle=-\langle q,v\rangle$ into Eq.~\ref{eq:app_gyr_raw}:
\begin{align*}
    a &= -c^2\langle z,v\rangle\|q\|^2-c\langle q,v\rangle+2c^2\langle z,q\rangle\langle q,v\rangle,\\
    b &= c^2\langle q,v\rangle\|z\|^2-c\langle z,v\rangle .
\end{align*}
The two terms of order $c$ combine into $c\big(\langle z,v\rangle-\langle q,v\rangle\big)=-c\langle\delta,v\rangle$. For the remaining terms, expand $q=z+\delta$ in
\begin{equation*}
    c^{-2}\big(a-b\big)_{O(c^2)}
    = -\langle z,v\rangle\|q\|^2-\langle q,v\rangle\|z\|^2+2\langle z,q\rangle\langle q,v\rangle .
\end{equation*}
Using $\|q\|^2=\|z\|^2+2\langle z,\delta\rangle+\|\delta\|^2$, $\langle q,v\rangle=\langle z,v\rangle+\langle\delta,v\rangle$ and $\langle z,q\rangle=\|z\|^2+\langle z,\delta\rangle$, the coefficient of $\langle z,v\rangle\|z\|^2$ is $-1-1+2=0$ and the coefficient of $\langle z,v\rangle\langle z,\delta\rangle$ is $-2+2=0$, so both $O(1)$ contributions cancel identically. What survives is
\begin{equation*}
    \|z\|^2\langle\delta,v\rangle-\langle z,v\rangle\|\delta\|^2+2\langle z,\delta\rangle\langle\delta,v\rangle .
\end{equation*}
Multiplying by $c^2$ and adding $-c\langle\delta,v\rangle$ gives
\begin{equation*}
    a-b=-c\langle\delta,v\rangle+c^2\|z\|^2\langle\delta,v\rangle-c^2\langle z,v\rangle\|\delta\|^2+2c^2\langle z,\delta\rangle\langle\delta,v\rangle,
\end{equation*}
which is the second line of Eq.~\ref{eq:app_gyr_stable} after collecting $-c\langle\delta,v\rangle(1-c\|z\|^2)$. Finally,
\begin{align*}
    \tfrac{a-b}{2}(q+z)-\tfrac{a+b}{2}(q-z)
    &= \tfrac12\big(aq+az-bq-bz\big)-\tfrac12\big(aq-az+bq-bz\big)\\
    &= a\,z-b\,q = aA+bB .
\end{align*}
\qed
\end{proof}

Both expressions in Eq.~\ref{eq:app_gyr_stable} are sums of terms of comparable, small magnitude, so the $O(1)$ contributions cancel symbolically rather than in floating point. The combination $a+b$ is evaluated directly, as $a$ and $b$ share a sign and no cancellation arises there. The estimator is therefore mathematically identical to the closed form of Eq.~\ref{eq:app_gyr_raw} while remaining finite as $1-c\|z_e\|^2\to0$.

\subsection{Recovery of Euclidean Residual Quantization at Zero Curvature}
\label{app:curvature}

\begin{proposition}
\label{prop:curvature}
At $c=0$, GHRQ reduces to standard Euclidean residual vector quantization: the forward pass coincides with it exactly, and the d-HSTE backward operator equals $\tfrac14 I$, that is, the identity straight-through estimator up to a constant rescaling.
\end{proposition}

\begin{proof}
At $c=0$ Eq.~\ref{eq:bg_mobius_add} gives $x\oplus_0 y=x+y$, so $\oplus_0$ is commutative and associative and $\operatorname{gyr}[u,v]=\mathrm{Id}$ for all $u,v$ by Eq.~\ref{eq:bg_gyration}. The HRA recursion of Eq.~\ref{eq:hrq_residual} therefore becomes $r_i=-q_i+r_{i-1}=r_{i-1}-q_i$ and $\hat z=q_1+\cdots+q_N=\sum_i q_i$, which is the Euclidean cascade; the association order is immaterial, so HRA and the naive convention coincide. The conformal factor is the constant $\lambda^0_x=2$, so $\widetilde P^0_{q\to z_e}=\tfrac{1}{2\cdot2}\mathrm{Id}=\tfrac14 I$. Proposition~\ref{prop:leak} gives $A_i=0$, so block-level routing and the per-stage cascade deliver the same encoder gradient.\qed
\end{proof}

In practice we use the identity estimator directly at $c=0$, so the Euclidean baseline is recovered exactly rather than up to the factor $\tfrac14$. Hyperbolic residual quantization is in this sense a strict generalization of its Euclidean counterpart rather than a separate model.

\end{document}